\documentclass[twoside]{article}

\usepackage[preprint]{aistats2027}
\usepackage[round]{natbib}
\usepackage{algorithm}
\usepackage{algpseudocode}
\usepackage{amsthm}
\newtheorem{prop}{Proposition}
\newtheorem*{prop*}{Proposition}
\usepackage{pgfplots} 
\usepgfplotslibrary{statistics}
\usetikzlibrary{arrows.meta}
\pgfplotsset{compat=1.18}

\usepackage{amssymb}
\usepackage{array}
\usepackage{booktabs}
\usepackage{hyperref}

\usepackage{graphicx}
\usepackage{svg}
\usepackage{pgfplots}
\usepgfplotslibrary{groupplots}
\newtheorem{definition}{Definition}

\definecolor{elasticcolor}{RGB}{31,119,180}
\definecolor{loopedcolor}{HTML}{D9A05B}

\pgfplotsset{
    colormap={ckamap}{
        color(0cm)=(white);
        color(1cm)=(elasticcolor!70)
    },
    colormap={ckadifference}{
        color(0cm)=(red!70!black);
        color(1cm)=(white);
        color(2cm)=(elasticcolor)
    }
}

\begin{document}

\runningtitle{I use this title instead because the last one was very long}

\runningauthor{Mieszko Komisarczyk, Saurabh Mathur, Maurice Kraus, Sriraam Natarajan, Kristian Kersting}

\twocolumn[

\aistatstitle{{ElTab}: Elastically Looped Tabular Transformers}

\aistatsauthor{
Mieszko Komisarczyk$^{1,*}$,
Saurabh Mathur$^{1, *}$,
Maurice Kraus$^{1}$,
Sriraam Natarajan$^{2}$,
Kristian Kersting$^{1,3,4}$
}

\aistatsaddress{
$^{1}$ Department of Computer Science,
Technical University of Darmstadt, Germany \\
$^{2}$ Department of Computer Science,
University of Texas at Dallas, USA \\
$^{3}$ Hessian Center for Artificial Intelligence
(hessian.AI), Darmstadt, Germany \\
$^{4}$ German Research Center for Artificial
Intelligence (DFKI), Darmstadt, Germany \\[5pt]
$^{*}$ \textit{These authors contributed equally.}
}
]

\begin{abstract}
 Tabular transformer models such as TabPFN achieve strong predictive performance but incur substantial inference costs, making them difficult to deploy in resource-constrained settings. While looped transformers reduce parameter count through weight sharing, they can lose significant predictive power when stopped early. To address this challenge, we introduce the Elastically Looped Tabular Transformer (ElTab), a single model trained to operate across different computational budgets through supervision at randomly sampled early exits. Across 30 OpenML classification datasets, a single ElTab model matches six depth-specific looped models while cutting training cost by more than half, and matches a 12-layer TabPFN-style baseline with 25\% fewer inference FLOPs. These results show that elastic-depth training yields a single model that adapts to the deployment budget at hand, offering a favorable trade-off among predictive performance, storage, and computational cost.
\end{abstract}

\section{Introduction}

Tabular foundation models such as TabPFN\citep{hollmann2025accurate} have emerged as general-purpose models for tabular prediction, achieving strong predictive performance across diverse datasets. They predict on a new dataset in a single forward pass, conditioning a pretrained Transformer on the dataset's labeled examples instead of fitting dataset-specific parameters. However, this in-context prediction is expensive, as its cost grows with both backbone depth and attention that is quadratic in the number of context rows. Since depth is fixed at training time, a resource-constrained deployment cannot trade accuracy for compute at inference; its only option is to train a smaller model.

Looped Transformers~\citep{giannou2023looped} reduce the parameter count while preserving effective depth by repeatedly applying a shared stack of blocks. In tabular models, a single looped block has recently been shown to match a six-layer nanoTabPFN~\citep{balef2026one}. However, weight sharing alone does not make a model computationally adaptive. While a looped model can be stopped early, its intermediate representations are never trained to support predictions. In our experiments (Figure~\ref{fig:q1}), stopping a fixed-depth looped model after three of its six loops costs 7.3 AUC-ROC points on average across 30 OpenML datasets.

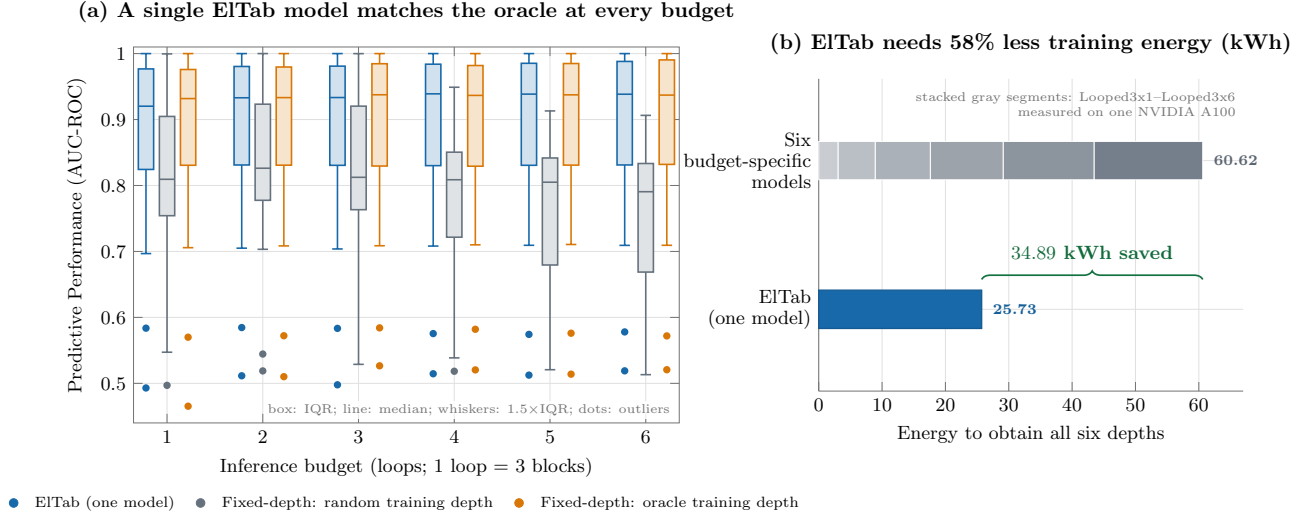
\begin{figure*}[t]
    \centering
    % ElTab performance--training-energy trade-off.
%
% Required in the paper preamble:
%   \usepackage{graphicx}
%   \usepackage{tikz}
%   \usepackage{pgfplots}
%   \usetikzlibrary{arrows.meta,calc,decorations.pathreplacing}
%   \pgfplotsset{compat=1.18}
%   \usepgfplotslibrary{statistics}
%
% Data sources (values are embedded to keep the figure self-contained):
%   result_csvs/all_looped_3x_depths_1_to_6_elastic_tabpfn_openml30_per_dataset.csv
%   result_csvs/training_energy_looped_3x_vs_elastic_raw.csv
%   result_csvs/training_energy_looped_3x_vs_elastic_summary.csv

\usepgfplotslibrary{statistics}

  \centering
  \definecolor{eltabblue}{HTML}{1769AA}
  \definecolor{fixedgray}{HTML}{66717E}
  \definecolor{specialistorange}{HTML}{D97706}
  \definecolor{savingsgreen}{HTML}{18864B}

  \resizebox{\textwidth}{!}{%
  \begin{tikzpicture}[font=\small]

    % -----------------------------------------------------------------
    % (a) Per-dataset performance distributions at each inference depth
    % -----------------------------------------------------------------
    \begin{axis}[
      name=performance,
      width=9.0cm,
      height=6.2cm,
      scale only axis,
      xmin=0.65,
      xmax=6.35,
      ymin=0.44,
      ymax=1.01,
      xtick={1,2,3,4,5,6},
      ytick={0.5,0.6,0.7,0.8,0.9,1.0},
      scaled y ticks=false,
      y tick label style={/pgf/number format/fixed,
                          /pgf/number format/precision=2},
      xlabel={Inference budget (loops; 1 loop = 3 blocks)},
      ylabel={Predictive Performance (AUC-ROC)},
      title={(a) A single ElTab model matches the oracle at every budget},
      title style={font=\bfseries,yshift=2pt},
      axis line style={black!55},
      tick style={black!55},
      grid=both,
      major grid style={black!12},
      minor grid style={black!6},
      minor x tick num=0,
      clip=false,
      boxplot/draw direction=y,
      boxplot/box extend=0.16,
      boxplot/whisker extend=0.11,
      legend style={
        at={(0.50,-0.17)},
        anchor=north,
        legend columns=3,
        draw=none,
        fill=none,
        font=\scriptsize,
        column sep=6pt,
        /tikz/every even column/.append style={column sep=8pt}
      }
    ]

      % Prepared Tukey box plots from the 30 per-dataset ROC--AUC values.
      % Arguments: position, lower whisker, Q1, median, Q3, upper whisker,
      %            color, and explicit outlier coordinates.
      \newcommand{\aucbox}[8]{%
        \addplot+[
          boxplot prepared={
            lower whisker=#2,
            lower quartile=#3,
            median=#4,
            upper quartile=#5,
            upper whisker=#6
          },
          boxplot/draw position=#1,
          forget plot,
          draw=#7,
          fill=#7!20,
          line width=0.75pt,
          solid,
          mark=*,
          mark size=1.35pt,
          mark options={fill=#7,draw=#7}
        ] coordinates {#8};%
      }

      % Depth 1.
      \aucbox{0.78}{0.6967554092}{0.8243254781}{0.9202135384}{0.9767554641}{1.0000000000}{eltabblue}
        {(0,0.4930042922) (0,0.5836659193)}
      \aucbox{1.00}{0.5472136299}{0.7542098264}{0.8093943278}{0.9046948870}{0.9995301922}{fixedgray}
        {(0,0.4969669461)}
      \aucbox{1.22}{0.7057118177}{0.8306503773}{0.9317875743}{0.9758734465}{1.0000000000}{specialistorange}
        {(0,0.4653538883) (0,0.5699553251)}

      % Depth 2.
      \aucbox{1.78}{0.7049637556}{0.8310952187}{0.9329963923}{0.9804547548}{1.0000000000}{eltabblue}
        {(0,0.5114739597) (0,0.5845966101)}
      \aucbox{2.00}{0.7032478293}{0.7775351226}{0.8261858493}{0.9232759953}{1.0000000000}{fixedgray}
        {(0,0.5188708196) (0,0.5444289108)}
      \aucbox{2.22}{0.7084049225}{0.8307804465}{0.9332516491}{0.9796430349}{1.0000000000}{specialistorange}
        {(0,0.5101301610) (0,0.5721772671)}

      % Depth 3.
      \aucbox{2.78}{0.7037687898}{0.8305565476}{0.9333531976}{0.9809667230}{1.0000000000}{eltabblue}
        {(0,0.4977958620) (0,0.5833584428)}
      \aucbox{3.00}{0.5288063814}{0.7633156081}{0.8124164740}{0.9202999135}{0.9999202867}{fixedgray}
        {}
      \aucbox{3.22}{0.7086466551}{0.8295211196}{0.9375414729}{0.9846338868}{1.0000000000}{specialistorange}
        {(0,0.5266000867) (0,0.5841290236)}

      % Depth 4.
      \aucbox{3.78}{0.7081556678}{0.8300886273}{0.9389732718}{0.9839246035}{1.0000000000}{eltabblue}
        {(0,0.5145492196) (0,0.5754610419)}
      \aucbox{4.00}{0.5386164546}{0.7215947608}{0.8087195143}{0.8505102019}{0.9489705821}{fixedgray}
        {(0,0.5183532099)}
      \aucbox{4.22}{0.7100215435}{0.8292415142}{0.9365899503}{0.9818930745}{1.0000000000}{specialistorange}
        {(0,0.5204102039) (0,0.5819969893)}

      % Depth 5.
      \aucbox{4.78}{0.7093255043}{0.8306032658}{0.9385397315}{0.9852691412}{1.0000000000}{eltabblue}
        {(0,0.5124869227) (0,0.5742333174)}
      \aucbox{5.00}{0.5206914475}{0.6794558704}{0.8049921547}{0.8417346934}{0.9130525331}{fixedgray}
        {}
      \aucbox{5.22}{0.7105460048}{0.8308714509}{0.9374347150}{0.9849271417}{1.0000000000}{specialistorange}
        {(0,0.5139376938) (0,0.5760122180)}

      % Depth 6.
      \aucbox{5.78}{0.7092062116}{0.8311247230}{0.9383881032}{0.9880951881}{1.0000000000}{eltabblue}
        {(0,0.5189821541) (0,0.5779796362)}
      \aucbox{6.00}{0.5131810943}{0.6687280521}{0.7905607951}{0.8333359051}{0.9063118458}{fixedgray}
        {}
      \aucbox{6.22}{0.7093184829}{0.8319809198}{0.9370593488}{0.9903906703}{1.0000000000}{specialistorange}
        {(0,0.5206532180) (0,0.5719299674)}

      \addlegendimage{only marks,mark=*,mark size=2pt,color=eltabblue}
      \addlegendentry{ElTab (one model)}
      \addlegendimage{only marks,mark=*,mark size=2pt,color=fixedgray}
      \addlegendentry{Fixed-depth: random training depth}
      \addlegendimage{only marks,mark=*,mark size=2pt,color=specialistorange}
      \addlegendentry{Fixed-depth: oracle training depth}

      \node[
        anchor=south east,
        font=\tiny,
        text=black!55,
        fill=white,
        fill opacity=0.85,
        text opacity=1,
        inner sep=1.5pt
      ] at (rel axis cs:0.99,0.01)
        {box: IQR; line: median; whiskers: $1.5\times$IQR; dots: outliers};
    \end{axis}

    % -----------------------------------------------------------------
    % (b) Aggregate energy required to obtain all six deployment depths
    % -----------------------------------------------------------------
    \begin{axis}[
      name=energy,
      at={($(performance.east)+(2.3cm,0)$)},
      anchor=west,
      width=7.0cm,
      height=5.15cm,
      scale only axis,
      xbar stacked,
      bar width=18pt,
      xmin=0,
      xmax=67,
      symbolic y coords={elastic,fixed},
      ytick=data,
      yticklabels={ElTab\\ (one model),{Six\\ budget-specific\\models}},
      y tick label style={font=\small,align=right},
      xlabel={Energy to obtain all six depths},
      title={(b) ElTab needs 58\% less training energy (kWh)},
      title style={font=\bfseries,yshift=2pt},
      axis x line*=bottom,
      axis y line*=left,
      axis line style={black!55},
      tick style={black!55},
      xtick={0,10,20,30,40,50,60},
      grid=major,
      major grid style={black!12},
      enlarge y limits=0.55,
      clip=false
    ]

      % Looped3x1,...,Looped3x6 form one stacked suite-level bar.
      \addplot+[draw=white,line width=0.45pt,fill=fixedgray!35]
        coordinates {(0,elastic) (3.042430,fixed)};
      \addplot+[draw=white,line width=0.45pt,fill=fixedgray!45]
        coordinates {(0,elastic) (5.871530,fixed)};
      \addplot+[draw=white,line width=0.45pt,fill=fixedgray!55]
        coordinates {(0,elastic) (8.698013,fixed)};
      \addplot+[draw=white,line width=0.45pt,fill=fixedgray!65]
        coordinates {(0,elastic) (11.495754,fixed)};
      \addplot+[draw=white,line width=0.45pt,fill=fixedgray!75]
        coordinates {(0,elastic) (14.348157,fixed)};
      \addplot+[draw=white,line width=0.45pt,fill=fixedgray!90]
        coordinates {(0,elastic) (17.163137,fixed)};

      % One elastically trained model supplies all six inference depths.
      \addplot+[draw=eltabblue!85!black,line width=0.6pt,fill=eltabblue]
        coordinates {(25.725834,elastic) (0,fixed)};

      \node[anchor=west,font=\scriptsize\bfseries,text=fixedgray!90!black]
        at (axis cs:61.2,fixed) {60.62};
      \node[anchor=west,font=\scriptsize\bfseries,text=eltabblue!90!black]
        at (axis cs:26.3,elastic) {25.73};

      % Keep the savings annotation in the gap between the bars and clear of the x-axis.
      \draw[
        decorate,
        decoration={brace,amplitude=4pt},
        savingsgreen!85!black,
        thick
      ] ([yshift=14pt]axis cs:25.725834,elastic) --
        node[
          midway,
          above=6pt,
          font=\small\bfseries,
          text=savingsgreen!80!black,
          align=center
        ] {$34.89$ kWh saved}
        ([yshift=14pt]axis cs:60.619020,elastic);

      \node[
        anchor=north east,
        font=\tiny,
        text=black!55,
        align=right
      ] at (rel axis cs:1.0,0.98)
        {stacked gray segments: Looped3x1--Looped3x6\\
         measured on one NVIDIA A100};
    \end{axis}

  \end{tikzpicture}%
  }
    \caption{\textbf{ElTab matches six budget-specific models, cutting training energy consumption by more than half.}
\textbf{(a)}~Per-dataset AUC-ROC on 30 OpenML datasets at each inference budget. Fixed-depth looped models are budget-specific, because their depth is set at training time, before the inference budget is known. This depth can be chosen at random (gray) or, as an idealized upper bound, by an oracle that knows the budget in advance (orange). A single ElTab model (blue), trained without knowing the budget, matches the oracle at every budget.
\textbf{(b)}~Estimated GPU energy for full training on one NVIDIA A100. Covering all six budgets with budget-specific models requires training six models (gray segments), which consume 60.6\,kWh in total. A single ElTab model covers the same budgets with 25.7\,kWh.}
    \label{fig:main}
\end{figure*}

\begin{figure*}[t]
\centering
\includegraphics[width=\linewidth]{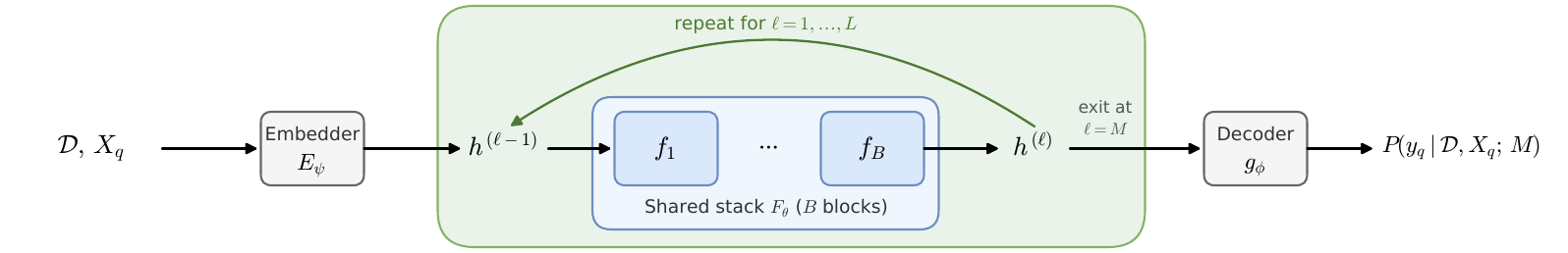}

% \caption{\textbf{Comparison of 12 looped model vs elastic one.}}
% \label{fig:elastic-vs-looped}
\caption{\textbf{ElTab architecture.} ElTab predicts the labels of query rows from a labeled context using three components. An embedder ($E_\psi$) turns the table into token representations. A shared stack of Transformer blocks ($F_\theta$) then refines these representations and can be applied repeatedly, up to $L$ times. Finally, a decoder ($g_\phi$) turns the refined query representations into class probabilities. Because ElTab is trained to predict after every loop, it can stop after any loop $M \le L$, letting the deployment choose how much compute to spend. All loops share the same weights and decoder.}
\end{figure*}

%whsiker plot, dots, depth oracle, LESS WHITE SPACE

To address this problem, we introduce ElTab, an \emph{Elastically  Looped Tabular Transformer,} whose recurrence depth can be adjusted to the computational budget available at inference time. At each training step, ElTab supervises both the full-depth output and one randomly sampled shallower exit, and distills the full-depth prediction into that exit. Thus, ElTab learns to predict after any loop, saving compute proportional to the context size for each skipped loop. To summarize, we make the following contributions:
\begin{itemize}
    \item We introduce ElTab, an elastically looped tabular foundation model that supports 6 inference budgets without retraining.

    % \item We propose an elastic-depth training objective that supervises a randomly sampled early exit and distills the full-depth prediction into it. 
    \item We propose an elastic-depth training objective that combines early-exit supervision with self-distillation.

    \item We evaluate ElTab on 30 OpenML classification benchmarks across 6 inference budget, showing that it matches an ensemble of 6 depth-specific looped models while reducing the training cost by 57\%. It matches the predictive performance of a 12-layer TabPFN-style baseline while using 25\% fewer inference FLOPs.
    
    \item We perform a centered kernel alignment study showing that elastic training makes latent representations similar across depths, demonstrating its robustness to truncation.
\end{itemize}

\section{Related work}

ElTab is related to three lines of research: tabular foundation models, looped Transformers, and methods for elastic inference budgets.  

\textbf{Tabular foundation models.} Tabular foundation models are pretrained neural networks that solve new tabular prediction tasks without task-specific training. They predict on a new dataset by conditioning the network on the dataset's labeled examples alongside the test queries, producing predictions in a single forward pass. The first such model, TabPFN~\citep{hollmann2022tabpfn,hollmann2025accurate}, is pretrained on millions of synthetic datasets sampled from a prior~\citep{muller2021transformers}, yielding strong predictive performance across diverse tabular benchmarks. Its architecture is based on the Transformer~\citep{vaswani2017attention}, with attention across rows that makes predictions invariant to row order. However, inference cost grows quadratically with the number of rows and linearly with the number of layers. Successive TabPFN versions have grown deeper~\citep{grinsztajn2025tabpfn}, raising inference cost and limiting their use in compute-constrained settings. One line of work lowers inference cost by replacing some or all Transformer layers with state-space models~\citep{koch2025state,komisarczyk2026tydra}, trading some of attention's expressiveness for efficiency. This trade-off is fixed by the architecture at training time. We instead repeatedly apply a few shared layers and train them elastically, allowing the trade-off between accuracy and compute to be chosen at inference time.

% Tabular foundation models cast tabular classification as an in-context learning problem. They condition a neural network on the entire training set alongside the test queries to produce predictions in a single forward pass. They are trained using synthetic datasets . Transformer-based Tabular foundation models, such as TabPFN have achieved  These models are based on the transformer architecture~\citep{}, attending over both rows and columns in every layer. As a result, the cost of each layer grows with the size of the context.  Successive versions have grown deeper, with TabPFN v2.5 roughly doubling the depth of v2~\citep{grinsztajn2025tabpfn}. TabPFN v3~\citep{grinsztajn2026tabpfn} still runs its full depth for every prediction. Since these models are pretrained on millions of synthetic tasks, supporting multiple inference compute budgets through separately trained models multiplies an already large training cost. ElTab addresses this challenge through a looped architecture that can be unrolled into additional layers based on available compute.

% Motivated by the need for more efficient inference, \citet{komisarczyk2026tydra} recently introduced a new line of hybrid tabular architectures that combine Transformer components with more computationally efficient sequence-processing mechanisms. ElTab takes a different but complementary approach by reducing the number of layers applied via looping.

\textbf{Looped Transformers.} Looped Transformers address the tension between the need for computational depth and parameter efficiency by applying the same layers repeatedly~\citep{giannou2023looped}. Instead of learning a deep stack of distinct layers, they learn a shallow stack and feed its output back as input over several iterations, refining the representation at no additional parameter cost. Applying a single block recurrently can match the performance of a deeper stack of distinct blocks~\citep{dehghani2018universal}. Small looped models also match or exceed much larger models on reasoning benchmarks such as ARC-AGI~\citep{jolicoeur2025less,komisarczyk2026recursive}. For tabular data, a single looped nanoTabPFN block has recently been shown to match a six-layer nanoTabPFN~\citep{balef2026one}. However, looping alone does not let a model adapt to the inference compute budget. Truncating a looped model trained at a fixed depth degrades its accuracy substantially (Section~\ref{sec:experiments}). ElTab addresses this limitation by training a looped TabPFN-style model to predict at any loop count.

\textbf{Elastic inference depth.} Deployment settings differ in the compute they can afford, yet training a separate model for each budget multiplies training and storage costs. Early-exit networks address this problem by attaching predictions to intermediate layers, allowing a single model to stop inference before the final layer~\citep{teerapittayanon2016branchynet,xin2020deebert,elhoushi2024layerskip}. Their shallower exits improve when trained to match the final prediction~\citep{zhang2019your}. The same principle has been applied to looped Transformers for latent reasoning, where conditioning each loop on its position and step size allows truncated models to degrade gracefully~\citep{jeddi2026loopformer}. ElTab extends elastic-depth training to tabular in-context learning.

% \textbf{Elastic computational budget.} Orthogonal to the choice of looping as an architecture, a separate line of work asks whether a model can be trained to operate under a variable, inference-time-selectable compute budget. \citet{frans2025one} introduce shortcut models, which condition a diffusion model on the desired step size and enforce self-consistency across trajectories of different lengths, enabling accurate generation anywhere from a single step to many, without retraining; \citet{xu2024expressive} give a theoretical account of the expressivity such step-conditioned shortcuts afford. \citet{jeddi2026loopformer} bring this idea to looped transformers for latent reasoning, conditioning each loop on both the current timestep and the chosen step size so that the same model degrades gracefully when truncated to fewer loops at inference. To the best of our knowledge, no such elastic, budget-conditioned training scheme has been investigated for looped tabular models.

\section{Elastically Looped Tabular Transformers}

% single tabular in-context learner 
% build a single tabular foundation model that can support 
% Given a new dataset and a budget, the model should predict the labels of the test queries without exceeding that budget. 
Tabular foundation models such as TabPFN are effective but expensive. Inference compute budgets may change dynamically, precluding its use. Training separate models for each budget is infeasible due to high training and storage costs. We aim to build a single model that can support tabular predictive inference at different inference compute budgets. We formalize this budget-constrained inference as follows,

\begin{center}
\fbox{\parbox{0.95\linewidth}{
\textbf{Given:} a labeled context $\mathcal{D} = \{(x_i, y_i)\}_{i=1}^{n}$, unlabeled queries $X_q$, and a compute budget $c \in \mathcal{C}$.\\[2pt]
\textbf{To Do:} Estimate the predictive probabilities $\{P(Y | x; \mathcal{D})\}_{x \in X_q }$ at a cost no greater than $c.$
}}
\end{center}

% Two obstacles make this goal difficult. First, a standard model can stop at an intermediate layer, but its intermediate representations are not trained for prediction. In tabular foundation models, these representations are informative yet misaligned with the final decoder, and predicting from them degrades accuracy~\citep{balef2026one}. Second, 
% Training a separate model for each budget avoids this problem but multiplies training and storage costs.

We address this problem by introducing the ElTab (Elastically Looped Tabular Transformers) family of models. ElTab models consist of a small stack of template layers, applied repeatedly at inference, and trained to predict after every repetition. The number of repetitions determines the cost and can be selected at deployment to fit the budget.

% \begin{figure*}[t]
%     \centering
%     \includegraphics[width=\linewidth]{AISTATS2027/Inference_Budget.pdf}
%     \caption{Caption}
%     \label{fig:inference_budget}
% \end{figure*}

\subsection{The ElTab Family of Models}

We formalize the ElTab family of models through three components. The \textbf{embedder} maps the context and the queries to token representations. The \textbf{template} stack refines these representations by applying the same layers repeatedly. Finally, the \textbf{decoder} maps the refined representations to class logits. The template stack is applied at most $L$ times, and the deployed model meets a budget by choosing how many times to apply it. We formalize ElTab as follows.

\begin{definition}[ElTab architecture]
Let $E_\psi$ denote the embedder, let
\begin{equation}
F_\theta^{[B]}
=
f_{\theta_B}\circ\cdots\circ f_{\theta_1}
\end{equation}
be a stack of $B$ Transformer blocks, and let $g_\phi$ denote the
decoder. Given a context set $\mathcal D$ and query inputs $X_q$, the
embedder produces the initial representations
\begin{equation}
h^{(0)}
=
E_\psi(\mathcal D,X_q).
\end{equation}

Let $L$ denote the maximum supported repetition depth. Repeated
applications of the stack produce
\begin{equation}
h^{(\ell)}
=
F_\theta^{[B]}\!\left(h^{(\ell-1)}\right),
\qquad
\ell=1,\ldots,L.
\end{equation}
The parameters $\theta=(\theta_1,\ldots,\theta_B)$ are shared across all
applications of the stack.

Let $h_q^{(\ell)}$ denote the query-token representations contained in
$h^{(\ell)}$. The prediction logits at depth $\ell$ are
\begin{equation}
z_\ell
=
g_\phi\!\left(h_q^{(\ell)}\right),
\qquad
\ell=1,\ldots,L.
\end{equation}
The same decoder $g_\phi$ is used at every depth.

Given a computational budget $c$, the inference depth is
\begin{equation}
M(c)
=
\max\left\{
m\in\{1,\ldots,L\}
\;\middle|\;
C_{\mathrm{ElTab}}(m)\leq c
\right\},
\end{equation}
assuming that the feasible set is nonempty. Under budget $c$, ElTab
applies the stack $M(c)$ times and returns $z_{M(c)}$.

The modules $(E_\psi,F_\theta^{[B]},g_\phi)$ define the ElTab
architecture, $L$ is its maximum repetition depth, and $M(c)$ is its
budget-dependent inference depth.
\end{definition} 
The model applies $B \cdot M$ blocks while storing the parameters of only $B$. Its inference cost therefore grows approximately linearly with $M$ at a fixed parameter count.

\begin{figure*}[t]
    \centering
    \includegraphics[width=\textwidth]{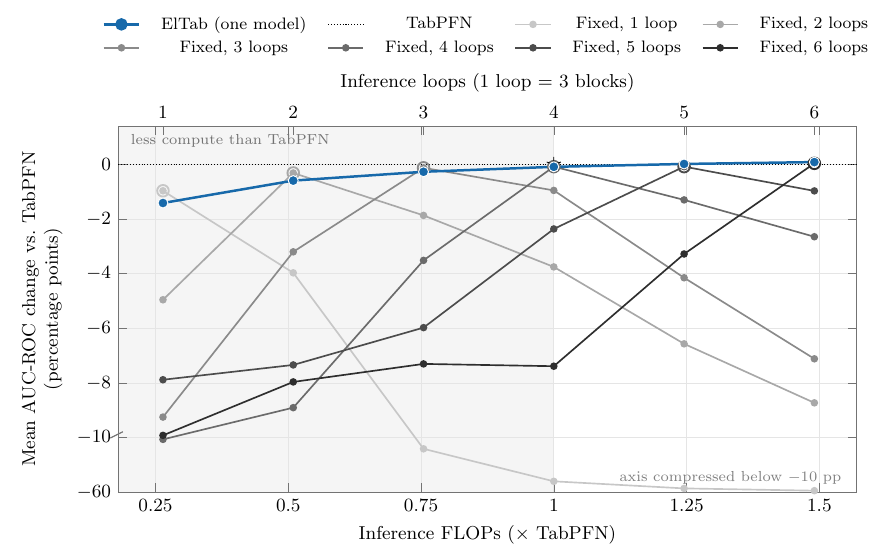}
    \caption{\textbf{Looped models match TabPFN with less compute at the budget they were trained for; ElTab does so across budgets.} Change in mean AUC-ROC on 30 OpenML datasets relative to TabPFN (dotted line; star at its own compute), plotted against inference compute. Gray lines show six fixed-depth looped models, each trained for one budget and evaluated after every loop; rings mark the budget each was trained for. Each fixed-depth model comes close to TabPFN at its own budget, even below TabPFN's compute, but drops sharply when run at any other budget. A single ElTab model (blue) stays close to TabPFN across budgets, matching it to within 0.27 points while using 25\% fewer FLOPs.}
    \label{fig:q1}
\end{figure*}

\subsection{Training}

ElTab is pretrained on synthetic classification tasks drawn from a prior
$\pi$, following the training protocol of recent TabPFN-style
models~\citep{koch2025state,komisarczyk2026tydra}. Each task consists of
a labeled context and a set of query examples. At every training step,
we sample a shortcut depth $k$ uniformly from $\{1,\ldots,L-1\}$ and
perform two separate forward passes through the same shared model: a
full-depth pass with $L$ applications of the block stack and a shortcut
pass with $k$ applications. Both predictions are supervised using the
query labels. In addition, the shortcut prediction is distilled from
the detached full-depth prediction, which serves as a fixed teacher for
the distillation term. Sampling a new shortcut depth at each step
provides supervision to every depth over the course of training. We
formalize this training objective below.

\begin{definition}[ElTab objective]
Let $\Theta = (\psi, \theta, \phi)$ collect all parameters and $\tau = (\mathcal{D}, X_q, y_q)$ denote a task. For a sampled depth $k$, the loss function is defined as
\begin{equation}
\begin{aligned}
\mathcal{L}(\Theta; \tau, k) ={}& \mathrm{CE}(y_q, z_L) + \lambda_{\mathrm{s}}\, \mathrm{CE}(y_q, z_k) \\ &+ \lambda_{\mathrm{d}}\, T^2\, \mathrm{KL}\big(\mathrm{sg}[p_L] \,\|\, p_k\big),
\end{aligned}
\end{equation}
where $p_\ell = \mathrm{softmax}(z_\ell / T)$ with temperature $T$, and $\mathrm{sg}[\cdot]$ denotes the stop-gradient operator. During model training, we minimize the expected loss
\begin{equation}
\mathbb{E}_{\tau \sim \pi,\, k \sim \mathcal{U}_L} \big[\mathcal{L}(\Theta; \tau, k)\big],
\end{equation}
where $\mathcal{U}_L$ is the uniform distribution over $\{1,\dots,L-1\}$.
\end{definition}

We use $\lambda_{\mathrm{s}} = \lambda_{\mathrm{d}} = 0.1$ and $T = 2$. These small weights keep the full-depth term dominant. Setting both weights to zero recovers standard fixed-depth training. We optimize with AdamW at a learning rate of $10^{-4}$, an effective batch size of 64, and gradient-norm clipping at 1.0 for 200 epochs. Algorithm~\ref{alg:train} summarizes the procedure.

% The first term preserves predictive performance at the maximum depth. The second term directly supervises the randomly selected shortcut, ensuring that predictions made under smaller computational budgets remain useful. Finally, the distillation term encourages the shortcut prediction to match the softer class distribution produced by the full-depth path. We detach the full-depth distribution in this term so that it acts as a teacher, while it continues to be optimized through its supervised cross-entropy loss.

\begin{algorithm}[H]
\caption{ElTab $B{\times}L$ Training}
\label{alg:train}
\begin{algorithmic}[1]
\Require Task prior $\pi$; embedder $E_{\psi}$;
shared stack $F_{\theta}^{[B]}$ of $B$ blocks;
decoder $g_{\phi}$; maximum loop depth $L$;
weights $\lambda_{\mathrm{s}},\lambda_{\mathrm{d}}$;
temperature $T$

\While{not converged}
    \State Sample $(\mathcal{D},X_q,y_q)\sim\pi$
    \State Sample
    $k\sim\mathcal{U}\!\left(\{1,\ldots,L-1\}\right)$

    \Statex
    \State $H_L^{(0)}
        \gets E_{\psi}(\mathcal{D},X_q)$
        \Comment{full-depth forward pass}
    \For{$\ell=1,\ldots,L$}
        \State $H_L^{(\ell)}
        \gets F_{\theta}^{[B]}\!\left(H_L^{(\ell-1)}\right)$
    \EndFor
    \State $z_L
        \gets g_{\phi}\!\left(H_{L,q}^{(L)}\right)$

    \Statex
    \State $H_k^{(0)}
        \gets E_{\psi}(\mathcal{D},X_q)$
        \Comment{separate shortcut forward pass}
    \For{$\ell=1,\ldots,k$}
        \State $H_k^{(\ell)}
        \gets F_{\theta}^{[B]}\!\left(H_k^{(\ell-1)}\right)$
    \EndFor
    \State $z_k
        \gets g_{\phi}\!\left(H_{k,q}^{(k)}\right)$

    \State $p_L^{(T)}
        \gets \operatorname{softmax}\!\left(z_L/T\right)$
    \State $p_k^{(T)}
        \gets \operatorname{softmax}\!\left(z_k/T\right)$

    \State $\mathcal{L}
        \gets \operatorname{CE}(y_q,z_L)
        + \lambda_{\mathrm{s}}\operatorname{CE}(y_q,z_k)$
    \Statex \hspace{\algorithmicindent}\hspace{\algorithmicindent}
        $\displaystyle
        + \lambda_{\mathrm{d}}T^2
        D_{\mathrm{KL}}\!\left(
        \operatorname{sg}[p_L^{(T)}]
        \,\middle\|\,
        p_k^{(T)}
        \right)$

    \State Update $\Theta=(\psi,\theta,\phi)$ using
        $\nabla_{\Theta}\mathcal{L}$
\EndWhile
\end{algorithmic}
\end{algorithm}

Because the shortcut depth is sampled uniformly at each training step,
every depth below $L$ receives supervision over the course of training.
The additional shortcut forward pass is used only during training. At
inference time, ElTab performs a single forward pass and stops after the
selected number of complete applications of the shared block stack.
Thus, a single set of learned parameters supports multiple
accuracy--compute operating points. Algorithm~\ref{alg:elastic-looped-inference}
summarizes the inference procedure. As the inference depth increases,
the computational cost grows linearly with the number of loops, whereas
the parameter count remains fixed. The following proposition formalizes
these properties.

\begin{algorithm}[H]
\caption{ElTab $B{\times}L$ Inference}
\label{alg:elastic-looped-inference}
\begin{algorithmic}
\Require Context $\mathcal{D}$; queries $X_q$; inference budget $M \in \{1,\ldots,L\}$; stack $F^{[B]}_{\theta}$ of $B$ blocks
\State $h^{(0)} \gets E_{\psi}(\mathcal{D}, X_q)$
\For{$\ell = 1,\ldots,M$}
    \State $h^{(\ell)} \gets F^{[B]}_{\theta}\!\left(h^{(\ell-1)}\right)$
\EndFor
\State \Return $g_{\phi}\!\left(h_{q}^{(M)}\right)$ \Comment{query class logits}
\end{algorithmic}
\end{algorithm}
\begin{prop}[Inference cost of ElTab]
\label{prop:cost}
Consider an ElTab model with a layer template stack $F_\theta$ of $B$ blocks and a maximum depth of $L$ loops. The model predicts the labels of queries $X_q$ from a context $\mathcal{D}$, with $N$ rows in total, and stops after $M \in \{1, \dots, L\}$ loops chosen to fit the inference budget. Denote the cost of this prediction by $C(M, N)$, the cost of one block on this input by $c(N)$, and the combined cost of the embedder $E_\psi$ and the decoder $g_\phi$ by $c_0(N)$. Then:
\begin{enumerate}
    \item The cost grows linearly with the number of loops: $C(M, N) = c_0(N) + BM\,c(N)$.
    \item If the embedder and decoder costs are negligible, stopping after $M$ loops costs a fraction $M/L$ of full depth: $C(M, N)/C(L, N) \approx M/L$.
   \item For hidden dimension $d$, the template layer stack has $\Theta(Bd^2)$ parameters, independent of $M$, $L$, and $N$.
    \item With dense attention among the $N$ rows and feed-forward width proportional to $d$, each block costs $c(N) = O(N^2 d + N d^2)$. Hence $C(M, N) = O\big(BM(N^2 d + N d^2)\big)$.
\end{enumerate}
\end{prop}
The proof is given in Appendix~\ref{app:proof}.

\section{Empirical evaluation}\label{sec:experiments}

%bigger figure, move the legend

%\begin{figure}[t]
    %\centering
    %\input{partials/per_dataset.tex} 
    %\label{fig:per_dataset}
    %\caption{A single ElTab model achieves 11\% better %AUROC on average across the 6 inference budgets, %without requiring retraining. }
%\end{figure}
\begin{table}[t]
    \centering
    \caption{
        Training computation and checkpoint storage for the fixed-depth
        looped models and ElTab. FLOPs include profiler-supported operations
        in the forward pass, loss computation, and backward pass. Total FLOPs
        are computed over 200 epochs with 4,672 microbatches per epoch. For
        ElTab, the reported computation is the expectation over uniformly
        sampled shortcut depths \(M \in \{1,\ldots,5\}\), together with the
        full-depth trajectory.
    }
    \label{tab:training_flops}
    \begin{tabular}{lrrr}
        \toprule
        Model
        & \shortstack[r]{FLOPs per\\batch\\(\(\times 10^{12}\))}
        & \shortstack[r]{Total\\FLOPs\\
        (\(\times 10^{18}\))
        }
        & \shortstack[r]{Ckpt.\\size\\(MB)} \\
        \midrule
        Looped3x1
            & \(0.415 \)
            & \(0.3874 \)
            & \(27\) MB \\
        Looped3x2
            & \(0.801 \)
            & \(0.7487\)
            & \(27\) MB \\
        Looped3x3
            & \(1.188\)
            & \(1.110\)
            & \(27\) MB \\
        Looped3x4
            & \(1.575 \)
            & \(1.471\)
            & \(27\) MB \\
        Looped3x5
            & \(1.961 \)
            & \(1.833\)
            & \(27\) MB \\
        Looped3x6
            & \(2.348 \)
            & \(2.194\)
            & \(27\) MB \\
        \midrule
        Looped total
            & \(8.288 \)
            & \(7.744\)
            & \(162\) MB \\
        \textbf{ElTab }
            & \(\mathbf{3.536}\)
            & \(\mathbf{3.304}\)
            & \(\mathbf{27}\) MB \\
        \bottomrule
    \end{tabular}
\end{table}
\begin{figure*}[t!]
   \centering
   \resizebox{\textwidth}{!}{\input{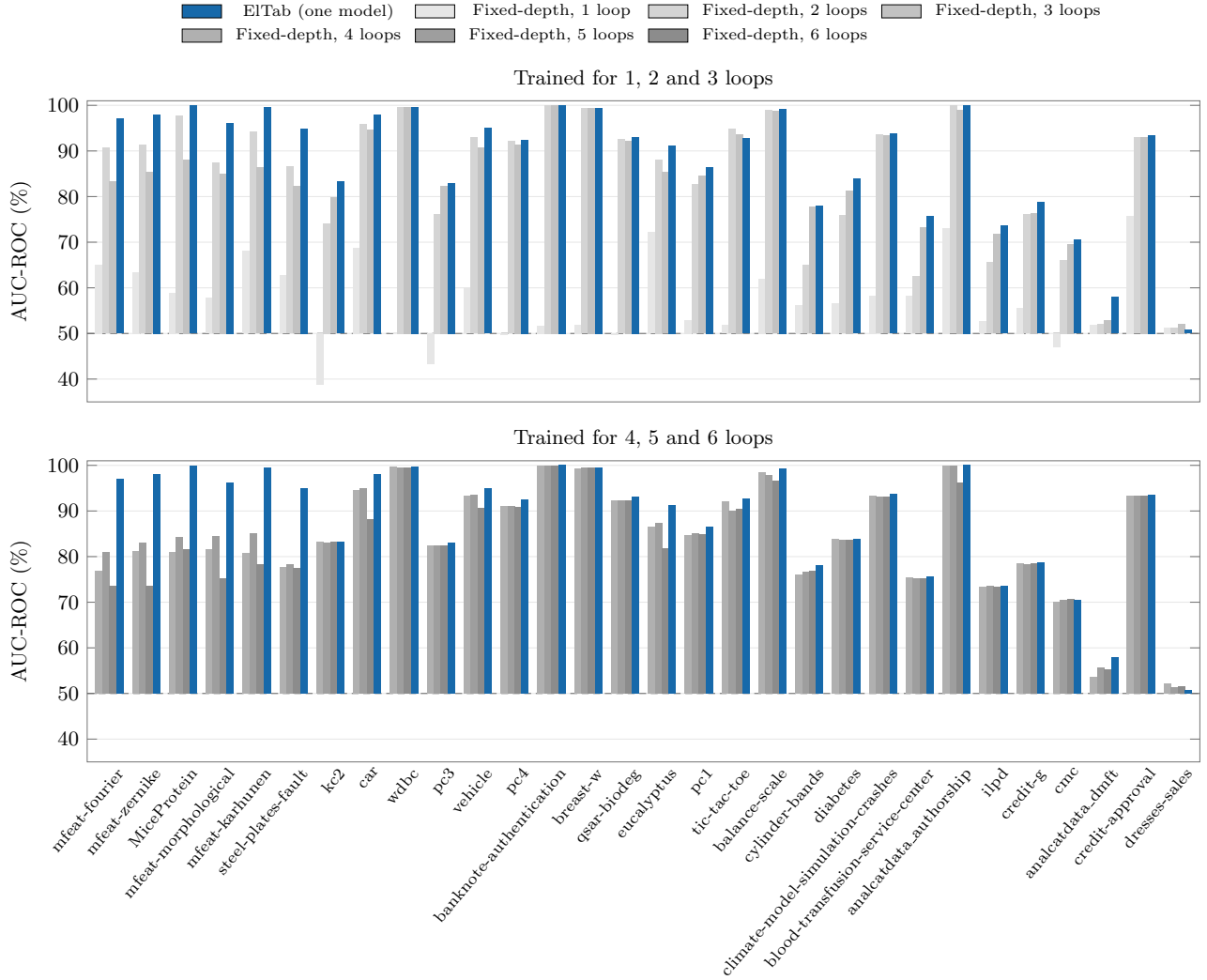}}
%     \caption{A single ElTab model achieves 11\% better AUROC on average across the 6 inference budgets, without requiring retraining}
\caption{\textbf{A single ElTab model outperforms every fixed-depth model across budgets on nearly all datasets.} AUC-ROC on each of 30 OpenML datasets, averaged over all six inference budgets, for ElTab and for fixed-depth looped models trained for one to six loops (top: one to three loops; bottom: four to six loops). Averaging over budgets reflects deploying one model at every budget, including those it was not trained for. Datasets are sorted by ElTab's gain over the mean of the six fixed-depth models. ElTab matches or outperforms all six fixed-depth models on 28 of 30 datasets. Its gains are largest where the fixed-depth models depend most on their training depth, such as mfeat-fourier and MiceProtein.}\label{fig:per_dataset}
\end{figure*}

We evaluate the accuracy, efficiency, and internal representations of ElTab through the following questions.

\begin{itemize}
    \item[(Q1)] Can looped models match the predictive performance of a standard TabPFN-style model at lower inference compute?
    \item[(Q2)] Can a single ElTab model match six independently trained fixed-depth looped models across inference budgets?
     \item[(Q3)] How does elastic training change the similarity of hidden representations across loop depths?
\end{itemize}

\subsection{Models}
To answer these questions, we compare ElTab with two kinds of baselines. ElTab loops a stack of three blocks up to six times and is trained with the elastic objective. The first baseline is a standard 12-block TabPFN-style model, as in~\citep{koch2025state}, which we refer to as TabPFN. The second is a set of depth-specific looped models (Looped3$\times$1 to Looped3$\times$6), each looping the same three-block stack a fixed number of times and trained only at that depth. All models use TabPFN-style blocks with identical dimensions~\citep{koch2025state}.

\subsection{Datasets}
We compare these models on the 30 binary and multiclass classification datasets from OpenML-CC18~\citep{bischl2017openml} that contain at most 2,000 observations, 100 features, and 10 classes, following~\citep{koch2025state}. Each dataset is split into a context and a query set. We report AUC-ROC on the query set, computed for multiclass datasets.

\subsection{Metrics}
We assessed each model based on predictive quality and inference efficiency. Predictive quality was quantified using the area under the receiver operating characteristic curve (AUC-ROC). On each dataset, results were averaged over five reproducible 50/50 train–test partitions, with every partition stratified to preserve all observed classes in both subsets. We present the mean AUC-ROC in the main text and report its standard deviation across splits in the appendix.

We quantify the energy savings of the elastic model relative to a
fixed-depth baseline as
\[
\boxed{
\text{Energy savings}
=
\frac{E_{\mathrm{fixed}}-E_{\mathrm{elastic}}}
     {E_{\mathrm{fixed}}}
\times 100\%
}
\]
where \(E_{\mathrm{fixed}}\) and \(E_{\mathrm{elastic}}\) denote the energy
consumed by the fixed-depth and elastic models, respectively, under the same
evaluation conditions. Positive values indicate that the elastic model
consumes less energy than the fixed-depth baseline, while negative values
indicate higher energy consumption. We compute the energy savings separately
for each dataset and experimental run and report the mean across datasets.

\textbf{(Answer, Q1: Looped vs.\ TabPFN)} To evaluate whether looped models can match TabPFN while spending less compute at inference, we compare TabPFN with the six fixed-depth looped models and ElTab on all 30 datasets. Because each loop applies three blocks and TabPFN applies twelve, four loops match TabPFN's compute, and fewer loops use less. Figure~\ref{fig:q1} shows how AUC-ROC changes with inference compute for each model. Each fixed-depth model comes close to TabPFN at the budget it was trained for, even below TabPFN's compute. For example, the model trained for three loops matches TabPFN to within 0.13 points while using 25\% fewer FLOPs. However, each fixed-depth model drops sharply when run at other budgets, by up to 58.6 points. In contrast, ElTab stays close to TabPFN across budgets, matching it to within 0.27 points at the same 25\% saving. Therefore, looped models can match the predictive performance of TabPFN at lower inference compute.

% \textbf{(Answer, Q1: Looped vs.\ TabPFN)} To evaluate whether a looped model can match TabPFN at lower inference compute, we compare TabPFN, ElTab, and the depth-specific looped models on all 30 datasets. Figure~\ref{fig:new} plots the change in mean ROC-AUC relative to TabPFN against inference FLOPs. At matched compute, ElTab applies its three-block stack four times and stays within [X] points of TabPFN while storing a quarter of its block parameters. At three loops, ElTab uses 25\% fewer FLOPs than TabPFN and loses only 0.27 points of ROC-AUC. The depth-specific Looped3$\times$6 model, truncated to the same budget, loses [X] points. Figure~\ref{fig:inference_budget} shows that this holds across datasets, not only on average, with median ROC-AUC staying close to TabPFN's down to [budget]. Sharing parameters across loops therefore does not limit accuracy. Elastic training extends this result to budgets below TabPFN's, answering Q1 affirmatively.

% \subsection{(Answer Q2)}

\textbf{(Answer, Q2: Elastic vs.\ budget-specific)} To evaluate whether one elastic model can match a set of budget-specific models, we compare ElTab with the six fixed-depth looped models, Looped3$\times$1 to Looped3$\times$6. Since each fixed-depth model commits to a depth at training time, we consider two ways of choosing that depth: (1) at random, and (2) an idealized method with an oracle that knows the inference budget in advance. Figure~\ref{fig:main}(a) shows the per-dataset AUC-ROC of ElTab and both choices at each budget. ElTab tracks the oracle at every budget, matching it to within 0.12 points on average and 0.45 points at most. In contrast, a randomly chosen depth falls well below both, because a fixed-depth model drops sharply when run at budgets for which it was not trained (Figure~\ref{fig:q1}). Figure~\ref{fig:per_dataset} confirms this pattern on individual datasets. ElTab matches or outperforms all six fixed-depth models on 28 of 30 datasets, on average across the six budgets. However, without an oracle, matching ElTab's accuracy at every budget would require training all six fixed-depth models. ElTab instead uses 57\% fewer training FLOPs and 58\% less GPU energy (Table~\ref{tab:training_flops}, Figure~\ref{fig:main}(b)), and needs one stored model instead of six. Therefore, a single ElTab model can match six fixed-depth looped models across inference budgets, without knowing the budget at training time.

\textbf{(Answer, Q3: Representations)} To evaluate how elastic training changes hidden representations, we compare ElTab with the fixed-depth Looped3$\times$6 model, which shares its architecture and maximum depth. For each model, we measure the similarity of hidden representations after each pair of loops $i, j \in \{1, \dots, 6\}$ using centered kernel alignment (CKA)~\citep{kornblith2019similarity}. In both models, representations at nearby loops are more similar than those at distant loops, consistent with gradual refinement across loops. Figure~\ref{fig:cka} shows how this similarity differs between the two models. Every off-diagonal entry is positive, and the difference is largest between the earliest loops and the final loop. These earliest loops are also where a truncated fixed-depth model loses the most accuracy. Elastic training therefore aligns representations most at the exits that need it most, consistent with ElTab's small accuracy loss when stopped early.

\begin{figure}[t]
    \centering
    \definecolor{eltabgreen}{HTML}{4f9e6b}

\begin{tikzpicture}
\begin{axis}[
    title={$\text{CKA}_\text{ElTab} - \text{CKA}_\text{Looped}$},
    width=5.5cm,
    height=5.8cm,
    xmin=0.5,
    xmax=6.5,
    ymin=0.5,
    ymax=6.5,
    xtick={1,2,3,4,5,6},
    ytick={1,2,3,4,5,6},
    xlabel={Depth $j$},
    ylabel={Depth $i$},
    y dir=normal,
    axis on top,
    enlargelimits=false,
    tick label style={font=\scriptsize},
    label style={font=\small},
    title style={font=\small},
    colormap={greenscale}{color(0cm)=(white); color(1cm)=(eltabgreen)},
    point meta min=0,
    point meta max=0.10,
    colorbar,
    colorbar style={
        width=0.15cm,
        ylabel={CKA difference},
        ylabel style={font=\scriptsize},
        tick label style={font=\tiny},
        ytick={0,0.025,0.05,0.075,0.10}
    }
]

\addplot[
    matrix plot*,
    point meta=explicit,
    draw=white,
    line width=0.25pt,
    mesh/cols=6
] coordinates {
    (1,1) [0.0000] (2,1) [0.0095] (3,1) [0.0230]
    (4,1) [0.0381] (5,1) [0.0639] (6,1) [0.0948]

    (1,2) [0.0095] (2,2) [0.0000] (3,2) [0.0118]
    (4,2) [0.0284] (5,2) [0.0490] (6,2) [0.0759]

    (1,3) [0.0230] (2,3) [0.0118] (3,3) [0.0000]
    (4,3) [0.0108] (5,3) [0.0309] (6,3) [0.0622]

    (1,4) [0.0381] (2,4) [0.0284] (3,4) [0.0108]
    (4,4) [0.0000] (5,4) [0.0132] (6,4) [0.0468]

    (1,5) [0.0639] (2,5) [0.0490] (3,5) [0.0309]
    (4,5) [0.0132] (5,5) [0.0000] (6,5) [0.0176]

    (1,6) [0.0948] (2,6) [0.0759] (3,6) [0.0622]
    (4,6) [0.0468] (5,6) [0.0176] (6,6) [0.0000]
};

\end{axis}
\end{tikzpicture} 
    \caption{\textbf{Elastic training keeps early-exit representations close to the final one.} Similarity of the hidden representations of ElTab after loops $i$ and $j$, as compared to an equivalent fixed-depth Looped3$\times$6 model. Similarity is measured as centered kernel alignment (CKA), which ranges from 0 (unrelated) to 1 (identical up to rotation and scaling). Positive values (green) mean that ElTab's representations at the two loops are more similar than those of Looped3$\times$6. All off-diagonal entries are positive, and the largest differences occur between the earliest loops and the final loop, showing that elastic training aligns representations most where early exits need it most.}\label{fig:cka}
\end{figure}
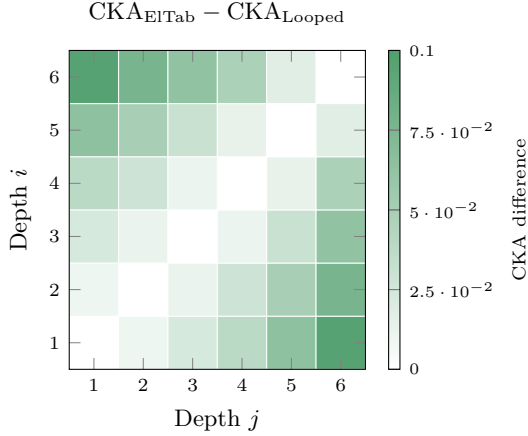

\section{Conclusions and Future Work}
In this work, we addressed the challenge of deploying tabular foundation models in settings where available compute is limited and variable. We introduced ElTab, a looped model trained to predict after every loop, allowing inference depth to be chosen based on the budget available at deployment, without requiring any retraining. One ElTab model matches TabPFN with 25\% less inference compute, and matches a separately trained model at every budget for less than half the training cost. Thus, ElTab's elastic training enables a single model to adapt to the compute budget available at deployment.

There are several directions for future work. ElTab's elastic training can be extended to adaptive computation by learning a stopping rule from its per-loop predictions, spending more compute on difficult queries and less on easy ones~\citep{bae2025mixture}. It can also be extended beyond loop counts to computational paths, by sampling which blocks are applied, skipped, or repeated during training~\citep{jeddi2026loopformer}. Finally, because the elastic objective adds no parameters, it can be extended to larger tabular foundation models by applying it to their looped backbones. Together with these directions, elastically trained looped models offer a favorable trade-off between predictive performance and inference cost for tabular foundation models deployed under limited and variable compute.

\subsection*{AI use statement}
In this work, we used generative AI tools to assist in developing the software implementation of the presented method, to create and edit the code that renders the presented figures, and to provide writing assistance to improve the clarity of the presentation. We have not used generative AI tools to generate, clean, or reformat any data sets; develop theoretical models or the method's conceptual frameworks; formulate, prove, or write down mathematical claims; propose or refine any hypotheses; assist with translation; or interpret scientific results. All AI-assisted work was carefully reviewed and verified by the authors for correctness, including checking all generated or edited code to ensure it correctly corresponds to the method presented in the paper. We take responsibility for the final content of this work, including text, claims, or artifacts produced with the aid of generative AI.
% In this work, we used generative AI tools to assist in developing the software implementation of the presented method and to provide writing assistance to improve the clarity of the presentation.
% We have not used generative AI tools to generate, clean, or reformat any data sets; develop theoretical models or the method's conceptual frameworks; formulate, prove, or write down mathematical claims; propose or refine any hypotheses; assist with translation; or interpret scientific results.
% Additionally, we used generative AI tools to create and edit the code that renders the presented figures.
% All AI-assisted work was carefully reviewed and verified by the authors for correctness.
% This includes checking all generated or edited code to ensure it correctly corresponds to the method presented in the paper.
% We take responsibility for the final content of this work, including text, claims, or artifacts produced with the aid of generative AI.

\subsection*{Acknowledgements}
We gratefully acknowledge support from the BMFTR project XEI (grant number 16IS24079B), the Cluster of Excellence ``Reasonable AI'' funded by the German Research Foundation (DFG) under Germany’s Excellence Strategy (EXC-3057), the DYNAMIC center funded by the LOEWE program of the Hessian Ministry of Science and Arts (grant number LOEWE1/16/519/03/09.001(0009)/98), German Federal Ministry for Economic Affairs and Energy (BMWE) through EU-SAI (grant number 13IPC040G), and the US Army Research Office (award W911NF2010224).

\bibliography{bibliography}
\clearpage

\appendix
\thispagestyle{empty}

% Supplementary material: To improve readability, you must use a single-column format for the supplementary material.
\onecolumn
\aistatstitle{                                                                                                     
Supplementary Materials}
\section{Experimental Setup}

\subsection{Implementation}  Our code is available at \url{https://anonymous.4open.science/r/ElTab--4684}.

\subsection{Hardware and Environment} All models were trained and evaluated on a single A100 GPU with 80GB memory, ensuring a consistent compute environment across runs.  Datasets for each training run were independently sampled from the same prior, so results remain comparable across models in terms of data distribution.

\section{Proofs} \label{app:proof}

\begin{prop*}[Inference cost of ElTab]
\label{prop:cost}
Consider an ElTab model with a layer template stack $F_\theta$ of $B$ blocks and a maximum depth of $L$ loops. The model predicts the labels of queries $X_q$ from a context $\mathcal{D}$, with $N$ rows in total, and stops after $M \in \{1, \dots, L\}$ loops chosen to fit the inference budget. Denote the cost of this prediction by $C(M, N)$, the cost of one block on this input by $c(N)$, and the combined cost of the embedder $E_\psi$ and the decoder $g_\phi$ by $c_0(N)$. Then:
\begin{enumerate}
    \item The cost grows linearly with the number of loops: $C(M, N) = c_0(N) + BM\,c(N)$.
    \item If the embedder and decoder costs are negligible, stopping after $M$ loops costs a fraction $M/L$ of full depth: $C(M, N)/C(L, N) = M/L$.
   \item For hidden dimension $d$, the template layer stack has $\Theta(Bd^2)$ parameters, independent of $M$, $L$, and $N$.
    \item With dense attention among the $N$ rows and feed-forward width proportional to $d$, each block costs $c(N) = O(N^2 d + N d^2)$. Hence $C(M, N) = O\big(BM(N^2 d + N d^2)\big)$.
\end{enumerate}
\end{prop*}

\begin{proof}
\begin{enumerate}
    \item We use induction on $B$. Denote by $C_B(M,N)$ the cost of prediction with $B$ blocks.
    For $B=1$, the cost is that of an $M$-looped model with a single block, namely $C_1(M,N)=c_0(N)+Mc(N)$.
    Now assume the claim holds for $B$ and add a $(B+1)$-th block after the first $B$ blocks. In each of the $M$ loops, this block adds a cost of $c(N)$. Hence
    \begin{equation} \label{eq:prop}
        C_{B+1}(M,N)=C_B(M,N)+Mc(N)=c_0(N)+BMc(N)+Mc(N)=c_0(N)+(B+1)Mc(N).
    \end{equation}
    \item From Equation~\ref{eq:prop} we obtain
    \begin{equation}
        C(M,N)=c_0(N)+BMc(N), \qquad C(L,N)=c_0(N)+BLc(N).
    \end{equation}
    Since $c_0(N)$ is negligible,
    \begin{equation}
        \frac{C(M,N)}{C(L,N)}=\frac{c_0(N)+BMc(N)}{c_0(N)+BLc(N)}\approx\frac{BMc(N)}{BLc(N)}=\frac{M}{L}.
    \end{equation}
    \item Each transformer-based block has $\Theta(d^2)$ parameters, and the weights are shared across loops, so the loop count $M$ (and $L$) does not affect the parameter count. Therefore $B$ blocks yield $\Theta(Bd^2)$ parameters, independent of $M$, $L$, and $N$.
    \item Dense attention among $N$ rows costs $O(N^2d)$, and the projections and feed-forward layers cost $O(Nd^2)$, so $c(N)=O(N^2d+Nd^2)$. Assuming $c_0(N)$ is dominated by $O(N^2d)$, Equation~\ref{eq:prop} gives $C(M,N)=O\big(BM\,c(N)\big)=O\big(BM(N^2d+Nd^2)\big)$.
\end{enumerate}
\end{proof}

\subsection{Datasets}For evaluation of our models, we considered the 30 datasets from the OpenML-CC18 benchmark suite~\cite{bischl2017openml} with at most 2{,}000 samples, 100 features, and 10 classes, following \cite{hollmann2022tabpfn} and \cite{koch2025state}. Table~\ref{datasets} presents these datasets along with their characteristics.

\subsection{Results Evaluation}
Table ~\ref{tab:eltab_detail} presents the detailed AUC-ROC results for  different depths of ElTab and TabPFN across all datasets. Table ~\ref{tab:looped} presents the detailed AUC-ROC results for all looped models and TabPFN across all datasets.
\begin{table*}[t]
\centering
\caption{
Characteristics of the 30 OpenML classification datasets, filtered for TabPFN limitations.
}
\label{tab:openml-dataset-characteristics}

\scriptsize
\setlength{\tabcolsep}{3.2pt}
\renewcommand{\arraystretch}{1.04}

\begin{tabular}{
    @{}
    >{\raggedright\arraybackslash}p{8.7cm}
    rrrrrrr
    @{}
}
\toprule
\textbf{Name}
&
\shortstack{\textbf{\#}\\\textbf{Feat.}}
&
\shortstack{\textbf{\#}\\\textbf{Cat.}}
&
\shortstack{\textbf{\#}\\\textbf{Inst.}}
&
\shortstack{\textbf{\#}\\\textbf{Class.}}
&
\shortstack{\textbf{\#}\\\textbf{NaNs}}
&
\shortstack{\textbf{Min. Class}\\\textbf{Size}}
&
\shortstack{\textbf{OpenML}\\\textbf{ID}}
\\
\midrule

balance-scale
& 5 & 1 & 625 & 3 & 0 & 49 & 11 \\

mfeat-fourier
& 77 & 1 & 2000 & 10 & 0 & 200 & 14 \\

breast-w
& 10 & 1 & 699 & 2 & 16 & 241 & 15 \\

mfeat-karhunen
& 65 & 1 & 2000 & 10 & 0 & 200 & 16 \\

mfeat-morphological
& 7 & 1 & 2000 & 10 & 0 & 200 & 18 \\

mfeat-zernike
& 48 & 1 & 2000 & 10 & 0 & 200 & 22 \\

cmc
& 10 & 8 & 1473 & 3 & 0 & 333 & 23 \\

credit-approval
& 16 & 10 & 690 & 2 & 67 & 307 & 29 \\

credit-g
& 21 & 14 & 1000 & 2 & 0 & 300 & 31 \\

diabetes
& 9 & 1 & 768 & 2 & 0 & 268 & 37 \\

tic-tac-toe
& 10 & 10 & 958 & 2 & 0 & 332 & 50 \\

vehicle
& 19 & 1 & 846 & 4 & 0 & 199 & 54 \\

eucalyptus
& 20 & 6 & 736 & 5 & 448 & 105 & 188 \\

analcatdata\_authorship
& 71 & 1 & 841 & 4 & 0 & 55 & 458 \\

analcatdata\_dmft
& 5 & 5 & 797 & 6 & 0 & 123 & 469 \\

pc4
& 38 & 1 & 1458 & 2 & 0 & 178 & 1049 \\

pc3
& 38 & 1 & 1563 & 2 & 0 & 160 & 1050 \\

kc2
& 22 & 1 & 522 & 2 & 0 & 107 & 1063 \\

pc1
& 22 & 1 & 1109 & 2 & 0 & 77 & 1068 \\

banknote-authentication
& 5 & 1 & 1372 & 2 & 0 & 610 & 1462 \\

blood-transfusion-service-center
& 5 & 1 & 748 & 2 & 0 & 178 & 1464 \\

ilpd
& 11 & 2 & 583 & 2 & 0 & 167 & 1480 \\

qsar-biodeg
& 42 & 1 & 1055 & 2 & 0 & 356 & 1494 \\

wdbc
& 31 & 1 & 569 & 2 & 0 & 212 & 1510 \\

cylinder-bands
& 40 & 22 & 540 & 2 & 999 & 228 & 6332 \\

dresses-sales
& 13 & 12 & 500 & 2 & 835 & 210 & 23381 \\

MiceProtein
& 82 & 5 & 1080 & 8 & 1396 & 105 & 40966 \\

car
& 7 & 7 & 1728 & 4 & 0 & 65 & 40975 \\

steel-plates-fault
& 28 & 1 & 1941 & 7 & 0 & 55 & 40982 \\

climate-model-simulation-crashes
& 21 & 1 & 540 & 2 & 0 & 46 & 40994 \\

\bottomrule
\end{tabular}
\label{datasets}
\end{table*}

\begin{table*}[t]
\centering
\caption{Hyperparameters and architectural details of TabPFN, ElTab
$3{\times}6$, and Looped $3{\times}6$.}
\label{tab:elastic_looped_hyperparameters}
\resizebox{0.95\textwidth}{!}{%
\begin{tabular}{lccc}
\toprule
\textbf{Hyperparameter}
& \textbf{TabPFN}
& \textbf{ElTab $3{\times}6$}
& \textbf{Looped $3{\times}6$} \\
\midrule

Learning Rate
& $1\times10^{-5}$
& $1\times10^{-4}$
& $1\times10^{-4}$ \\

Dropout
& 0.0
& 0.0
& 0.0 \\

Effective Batch Size
& 64
& 64
& 64 \\

Microbatch Size
& 8
& 8
& 8 \\

Microbatches per Epoch
& 4672
& 4672
& 4672 \\

Optimizer Updates per Epoch
& 584
& 584
& 584 \\

Aggregate $k$ Gradients
& 8
& 8
& 8 \\

Training Epochs
& 200
& 200
& 200 \\

\midrule

Embedding Size
& 512
& 512
& 512 \\

FFN/Decoder Hidden Size
& 1024
& 1024
& 1024 \\

Physical Transformer Layers
& 12
& 3
& 3 \\

Core Repetitions
& \textemdash
& 6
& 6 \\

Maximum Effective Layer Applications
& 12
& 18
& 18 \\

Number of Attention Heads
& 4
& 4
& 4 \\

Recompute Attention
& Attention only
& Attention only
& Attention only \\

Use AMP
& False
& False
& False \\

Optimizer
& AdamW
& AdamW
& AdamW \\

Parameter Count
& 25.82M
& 6.90M
& 6.90M \\

\midrule

Training Passes per Microbatch
& One 12-layer pass
& \shortstack{One full 18-layer pass\\
  + one sampled shortcut pass}
& One 18-layer pass \\

Shortcut Depth Sampling
& \textemdash
& $d \sim \operatorname{Uniform}\{1,2,3,4,5\}$
& \textemdash \\

Shortcut Effective Depths
& \textemdash
& $\{3,6,9,12,15\}$
& \textemdash \\

Full-Depth CE Weight
& 1.0
& 1.0
& 1.0 \\

Shortcut CE Weight
& \textemdash
& 0.1
& \textemdash \\

Consistency/KL Weight
& \textemdash
& 0.1
& \textemdash \\

Distillation Temperature
& \textemdash
& 2.0
& \textemdash \\

\bottomrule
\end{tabular}%
}
\end{table*}

\begin{table*}[t]
\label{tab:eltab_detail}
\centering
\caption{AUCROC on the 30 OpenML classification datasets for ElTab  evaluated at inference depths 1--6, compared with TabPFN. Values are means over 5 splits with 95\% Student-$t$ confidence intervals. The best AUCROC (higher is better) in each row is bold.}
\label{tab:openml_auc_elastic}
\setlength{\tabcolsep}{3.2pt}
\renewcommand{\arraystretch}{1.05}
\scriptsize
\begin{tabular}{@{}rccccccc@{}}
\toprule
DID & ElTab $d{=}1$ & ElTab $d{=}2$ & ElTab $d{=}3$ & ElTab $d{=}4$ & ElTab $d{=}5$ & ElTab $d{=}6$ & TabPFN \\
\midrule
11 & 0.9781 $\pm$ 0.0078 & 0.9916 $\pm$ 0.0050 & 0.9946 $\pm$ 0.0048 & 0.9960 $\pm$ 0.0037 & 0.9965 $\pm$ 0.0031 & 0.9967 $\pm$ 0.0029 & \textbf{0.9974} $\pm$ 0.0021 \\
14 & 0.9588 $\pm$ 0.0034 & 0.9703 $\pm$ 0.0037 & 0.9726 $\pm$ 0.0045 & 0.9730 $\pm$ 0.0040 & 0.9731 $\pm$ 0.0046 & 0.9731 $\pm$ 0.0044 & \textbf{0.9735} $\pm$ 0.0041 \\
15 & 0.9923 $\pm$ 0.0010 & 0.9940 $\pm$ 0.0012 & 0.9944 $\pm$ 0.0010 & 0.9945 $\pm$ 0.0009 & 0.9946 $\pm$ 0.0008 & \textbf{0.9947} $\pm$ 0.0008 & 0.9942 $\pm$ 0.0009 \\
16 & 0.9902 $\pm$ 0.0055 & 0.9962 $\pm$ 0.0020 & 0.9968 $\pm$ 0.0017 & \textbf{0.9970} $\pm$ 0.0015 & 0.9969 $\pm$ 0.0017 & 0.9967 $\pm$ 0.0017 & 0.9964 $\pm$ 0.0019 \\
18 & 0.9531 $\pm$ 0.0061 & 0.9626 $\pm$ 0.0030 & 0.9631 $\pm$ 0.0036 & 0.9633 $\pm$ 0.0031 & 0.9634 $\pm$ 0.0028 & \textbf{0.9635} $\pm$ 0.0022 & 0.9629 $\pm$ 0.0033 \\
22 & 0.9736 $\pm$ 0.0059 & 0.9805 $\pm$ 0.0026 & 0.9810 $\pm$ 0.0019 & 0.9818 $\pm$ 0.0018 & 0.9823 $\pm$ 0.0020 & 0.9823 $\pm$ 0.0023 & \textbf{0.9826} $\pm$ 0.0021 \\
23 & 0.6968 $\pm$ 0.0123 & 0.7050 $\pm$ 0.0145 & 0.7038 $\pm$ 0.0156 & 0.7082 $\pm$ 0.0153 & 0.7093 $\pm$ 0.0147 & 0.7092 $\pm$ 0.0158 & \textbf{0.7114} $\pm$ 0.0145 \\
29 & 0.9339 $\pm$ 0.0124 & 0.9352 $\pm$ 0.0109 & 0.9347 $\pm$ 0.0113 & 0.9350 $\pm$ 0.0120 & 0.9352 $\pm$ 0.0121 & 0.9349 $\pm$ 0.0119 & \textbf{0.9354} $\pm$ 0.0122 \\
31 & 0.7894 $\pm$ 0.0086 & 0.7858 $\pm$ 0.0078 & 0.7881 $\pm$ 0.0057 & 0.7884 $\pm$ 0.0077 & 0.7881 $\pm$ 0.0089 & 0.7869 $\pm$ 0.0090 & \textbf{0.7896} $\pm$ 0.0074 \\
37 & 0.8371 $\pm$ 0.0161 & 0.8334 $\pm$ 0.0188 & 0.8397 $\pm$ 0.0164 & 0.8413 $\pm$ 0.0166 & 0.8414 $\pm$ 0.0173 & \textbf{0.8415} $\pm$ 0.0173 & 0.8413 $\pm$ 0.0164 \\
50 & 0.8385 $\pm$ 0.0146 & 0.8691 $\pm$ 0.0228 & 0.9320 $\pm$ 0.0215 & 0.9536 $\pm$ 0.0149 & 0.9804 $\pm$ 0.0089 & \textbf{0.9881} $\pm$ 0.0089 & 0.9751 $\pm$ 0.0094 \\
54 & 0.9261 $\pm$ 0.0080 & 0.9420 $\pm$ 0.0070 & 0.9533 $\pm$ 0.0044 & 0.9574 $\pm$ 0.0036 & 0.9596 $\pm$ 0.0033 & \textbf{0.9605} $\pm$ 0.0034 & 0.9572 $\pm$ 0.0038 \\
188 & 0.8971 $\pm$ 0.0088 & 0.9116 $\pm$ 0.0079 & 0.9146 $\pm$ 0.0107 & 0.9151 $\pm$ 0.0111 & \textbf{0.9157} $\pm$ 0.0108 & \textbf{0.9157} $\pm$ 0.0101 & 0.9148 $\pm$ 0.0111 \\
458 & 0.9999 $\pm$ 0.0001 & \textbf{1.0000} $\pm$ 0.0000 & \textbf{1.0000} $\pm$ 0.0000 & \textbf{1.0000} $\pm$ 0.0000 & \textbf{1.0000} $\pm$ 0.0000 & \textbf{1.0000} $\pm$ 0.0000 & \textbf{1.0000} $\pm$ 0.0000 \\
469 & 0.5837 $\pm$ 0.0084 & \textbf{0.5846} $\pm$ 0.0088 & 0.5834 $\pm$ 0.0164 & 0.5755 $\pm$ 0.0159 & 0.5742 $\pm$ 0.0161 & 0.5780 $\pm$ 0.0134 & 0.5755 $\pm$ 0.0109 \\
1049 & 0.8926 $\pm$ 0.0107 & 0.9282 $\pm$ 0.0173 & 0.9293 $\pm$ 0.0143 & 0.9312 $\pm$ 0.0145 & 0.9321 $\pm$ 0.0152 & 0.9326 $\pm$ 0.0147 & \textbf{0.9331} $\pm$ 0.0183 \\
1050 & 0.8263 $\pm$ 0.0347 & \textbf{0.8311} $\pm$ 0.0342 & 0.8306 $\pm$ 0.0342 & 0.8301 $\pm$ 0.0350 & 0.8306 $\pm$ 0.0348 & \textbf{0.8311} $\pm$ 0.0361 & 0.8238 $\pm$ 0.0278 \\
1063 & 0.8243 $\pm$ 0.0426 & 0.8339 $\pm$ 0.0404 & \textbf{0.8353} $\pm$ 0.0415 & 0.8338 $\pm$ 0.0422 & 0.8347 $\pm$ 0.0432 & 0.8347 $\pm$ 0.0438 & 0.8340 $\pm$ 0.0420 \\
1068 & 0.8246 $\pm$ 0.0377 & 0.8601 $\pm$ 0.0594 & 0.8755 $\pm$ 0.0533 & \textbf{0.8790} $\pm$ 0.0499 & 0.8758 $\pm$ 0.0497 & 0.8756 $\pm$ 0.0504 & 0.8716 $\pm$ 0.0548 \\
1462 & \textbf{1.0000} $\pm$ 0.0000 & \textbf{1.0000} $\pm$ 0.0000 & \textbf{1.0000} $\pm$ 0.0000 & \textbf{1.0000} $\pm$ 0.0000 & \textbf{1.0000} $\pm$ 0.0000 & \textbf{1.0000} $\pm$ 0.0000 & \textbf{1.0000} $\pm$ 0.0000 \\
1464 & 0.7579 $\pm$ 0.0156 & 0.7560 $\pm$ 0.0200 & 0.7557 $\pm$ 0.0185 & 0.7564 $\pm$ 0.0182 & 0.7578 $\pm$ 0.0179 & \textbf{0.7585} $\pm$ 0.0190 & 0.7564 $\pm$ 0.0217 \\
1480 & 0.7264 $\pm$ 0.0445 & 0.7347 $\pm$ 0.0464 & 0.7383 $\pm$ 0.0417 & 0.7376 $\pm$ 0.0407 & 0.7392 $\pm$ 0.0403 & \textbf{0.7399} $\pm$ 0.0398 & 0.7389 $\pm$ 0.0447 \\
1494 & 0.9237 $\pm$ 0.0131 & 0.9308 $\pm$ 0.0135 & 0.9320 $\pm$ 0.0144 & 0.9326 $\pm$ 0.0136 & 0.9326 $\pm$ 0.0139 & 0.9325 $\pm$ 0.0137 & \textbf{0.9339} $\pm$ 0.0135 \\
1510 & 0.9964 $\pm$ 0.0044 & 0.9964 $\pm$ 0.0048 & 0.9964 $\pm$ 0.0047 & 0.9964 $\pm$ 0.0047 & 0.9962 $\pm$ 0.0048 & 0.9963 $\pm$ 0.0048 & \textbf{0.9965} $\pm$ 0.0040 \\
6332 & \textbf{0.7935} $\pm$ 0.0289 & 0.7832 $\pm$ 0.0324 & 0.7800 $\pm$ 0.0337 & 0.7741 $\pm$ 0.0320 & 0.7743 $\pm$ 0.0321 & 0.7756 $\pm$ 0.0321 & 0.7815 $\pm$ 0.0274 \\
23381 & 0.4930 $\pm$ 0.0704 & 0.5115 $\pm$ 0.0802 & 0.4978 $\pm$ 0.0706 & 0.5145 $\pm$ 0.0993 & 0.5125 $\pm$ 0.1030 & \textbf{0.5190} $\pm$ 0.1001 & 0.5092 $\pm$ 0.0954 \\
40966 & 0.9994 $\pm$ 0.0004 & \textbf{1.0000} $\pm$ 0.0000 & \textbf{1.0000} $\pm$ 0.0000 & \textbf{1.0000} $\pm$ 0.0000 & \textbf{1.0000} $\pm$ 0.0000 & \textbf{1.0000} $\pm$ 0.0000 & \textbf{1.0000} $\pm$ 0.0000 \\
40975 & 0.9768 $\pm$ 0.0122 & 0.9754 $\pm$ 0.0148 & 0.9740 $\pm$ 0.0157 & 0.9839 $\pm$ 0.0106 & 0.9853 $\pm$ 0.0109 & 0.9855 $\pm$ 0.0119 & \textbf{0.9874} $\pm$ 0.0081 \\
40982 & 0.9453 $\pm$ 0.0168 & 0.9483 $\pm$ 0.0182 & 0.9484 $\pm$ 0.0174 & 0.9502 $\pm$ 0.0163 & 0.9503 $\pm$ 0.0163 & 0.9497 $\pm$ 0.0167 & \textbf{0.9553} $\pm$ 0.0132 \\
40994 & 0.9168 $\pm$ 0.0165 & 0.9400 $\pm$ 0.0128 & \textbf{0.9430} $\pm$ 0.0093 & \textbf{0.9430} $\pm$ 0.0108 & 0.9419 $\pm$ 0.0083 & 0.9419 $\pm$ 0.0089 & 0.9411 $\pm$ 0.0090 \\
\midrule
Mean & 0.8748 & 0.8830 & 0.8863 & 0.8881 & 0.8891 & \textbf{0.8898} & 0.8890 \\
\bottomrule
\end{tabular}
\end{table*}

\begin{table*}[t]
\label{tab:looped}
\centering
\caption{AUCROC on the 30 OpenML classification datasets for fixed-depth looped models (3$\times$1 to 3$\times$6) compared with TabPFN. Values are means over 5 splits with 95\% Student-$t$ confidence intervals. The best AUCROC (higher is better) in each row is bold.}
\label{tab:openml_auc_fixed}
\setlength{\tabcolsep}{3.2pt}
\renewcommand{\arraystretch}{1.05}
\scriptsize
\begin{tabular}{@{}rccccccc@{}}
\toprule
DID & Looped 3$\times$1 & Looped 3$\times$2 & Looped 3$\times$3 & Looped 3$\times$4 & Looped 3$\times$5 & Looped 3$\times$6 & TabPFN \\
\midrule
11 & 0.9908 $\pm$ 0.0059 & 0.9940 $\pm$ 0.0043 & 0.9959 $\pm$ 0.0036 & 0.9964 $\pm$ 0.0034 & 0.9971 $\pm$ 0.0023 & \textbf{0.9976} $\pm$ 0.0027 & 0.9974 $\pm$ 0.0021 \\
14 & 0.9676 $\pm$ 0.0036 & 0.9726 $\pm$ 0.0035 & 0.9729 $\pm$ 0.0045 & 0.9725 $\pm$ 0.0044 & 0.9730 $\pm$ 0.0046 & 0.9734 $\pm$ 0.0045 & \textbf{0.9735} $\pm$ 0.0041 \\
15 & 0.9935 $\pm$ 0.0010 & \textbf{0.9945} $\pm$ 0.0009 & 0.9944 $\pm$ 0.0005 & 0.9942 $\pm$ 0.0008 & \textbf{0.9945} $\pm$ 0.0008 & \textbf{0.9945} $\pm$ 0.0008 & 0.9942 $\pm$ 0.0009 \\
16 & 0.9943 $\pm$ 0.0030 & 0.9958 $\pm$ 0.0021 & 0.9961 $\pm$ 0.0019 & 0.9967 $\pm$ 0.0017 & 0.9967 $\pm$ 0.0017 & \textbf{0.9969} $\pm$ 0.0018 & 0.9964 $\pm$ 0.0019 \\
18 & 0.9591 $\pm$ 0.0034 & \textbf{0.9634} $\pm$ 0.0034 & \textbf{0.9634} $\pm$ 0.0033 & 0.9630 $\pm$ 0.0025 & 0.9626 $\pm$ 0.0035 & \textbf{0.9634} $\pm$ 0.0026 & 0.9629 $\pm$ 0.0033 \\
22 & 0.9759 $\pm$ 0.0052 & 0.9796 $\pm$ 0.0036 & 0.9813 $\pm$ 0.0022 & 0.9819 $\pm$ 0.0016 & 0.9819 $\pm$ 0.0013 & 0.9815 $\pm$ 0.0014 & \textbf{0.9826} $\pm$ 0.0021 \\
23 & 0.7057 $\pm$ 0.0142 & 0.7084 $\pm$ 0.0147 & 0.7086 $\pm$ 0.0138 & 0.7100 $\pm$ 0.0126 & 0.7105 $\pm$ 0.0138 & 0.7093 $\pm$ 0.0133 & \textbf{0.7114} $\pm$ 0.0145 \\
29 & 0.9339 $\pm$ 0.0115 & 0.9340 $\pm$ 0.0119 & 0.9352 $\pm$ 0.0134 & 0.9342 $\pm$ 0.0119 & 0.9344 $\pm$ 0.0123 & 0.9348 $\pm$ 0.0131 & \textbf{0.9354} $\pm$ 0.0122 \\
31 & 0.7879 $\pm$ 0.0064 & 0.7880 $\pm$ 0.0074 & 0.7873 $\pm$ 0.0081 & 0.7863 $\pm$ 0.0080 & 0.7865 $\pm$ 0.0083 & 0.7862 $\pm$ 0.0087 & \textbf{0.7896} $\pm$ 0.0074 \\
37 & 0.8311 $\pm$ 0.0188 & 0.8419 $\pm$ 0.0129 & \textbf{0.8426} $\pm$ 0.0167 & 0.8411 $\pm$ 0.0167 & 0.8424 $\pm$ 0.0157 & 0.8412 $\pm$ 0.0173 & 0.8413 $\pm$ 0.0164 \\
50 & 0.8733 $\pm$ 0.0179 & 0.9311 $\pm$ 0.0139 & 0.9538 $\pm$ 0.0118 & 0.9727 $\pm$ 0.0092 & 0.9692 $\pm$ 0.0144 & \textbf{0.9904} $\pm$ 0.0073 & 0.9751 $\pm$ 0.0094 \\
54 & 0.9395 $\pm$ 0.0080 & 0.9473 $\pm$ 0.0068 & 0.9530 $\pm$ 0.0056 & 0.9592 $\pm$ 0.0038 & 0.9600 $\pm$ 0.0037 & \textbf{0.9617} $\pm$ 0.0041 & 0.9572 $\pm$ 0.0038 \\
188 & 0.9051 $\pm$ 0.0101 & 0.9133 $\pm$ 0.0092 & 0.9147 $\pm$ 0.0112 & 0.9136 $\pm$ 0.0100 & \textbf{0.9149} $\pm$ 0.0108 & \textbf{0.9149} $\pm$ 0.0105 & 0.9148 $\pm$ 0.0111 \\
458 & \textbf{1.0000} $\pm$ 0.0001 & \textbf{1.0000} $\pm$ 0.0000 & \textbf{1.0000} $\pm$ 0.0000 & \textbf{1.0000} $\pm$ 0.0000 & \textbf{1.0000} $\pm$ 0.0000 & \textbf{1.0000} $\pm$ 0.0000 & \textbf{1.0000} $\pm$ 0.0000 \\
469 & 0.5700 $\pm$ 0.0116 & 0.5722 $\pm$ 0.0189 & \textbf{0.5841} $\pm$ 0.0167 & 0.5820 $\pm$ 0.0150 & 0.5760 $\pm$ 0.0100 & 0.5719 $\pm$ 0.0130 & 0.5755 $\pm$ 0.0109 \\
1049 & 0.9222 $\pm$ 0.0143 & 0.9299 $\pm$ 0.0139 & 0.9302 $\pm$ 0.0164 & 0.9310 $\pm$ 0.0167 & 0.9308 $\pm$ 0.0172 & 0.9314 $\pm$ 0.0180 & \textbf{0.9331} $\pm$ 0.0183 \\
1050 & 0.8307 $\pm$ 0.0268 & 0.8316 $\pm$ 0.0330 & 0.8295 $\pm$ 0.0374 & 0.8292 $\pm$ 0.0348 & 0.8309 $\pm$ 0.0365 & \textbf{0.8320} $\pm$ 0.0368 & 0.8238 $\pm$ 0.0278 \\
1063 & \textbf{0.8355} $\pm$ 0.0387 & 0.8308 $\pm$ 0.0382 & 0.8340 $\pm$ 0.0419 & 0.8327 $\pm$ 0.0400 & 0.8343 $\pm$ 0.0415 & 0.8350 $\pm$ 0.0416 & 0.8340 $\pm$ 0.0420 \\
1068 & 0.8587 $\pm$ 0.0582 & \textbf{0.8806} $\pm$ 0.0560 & 0.8727 $\pm$ 0.0524 & 0.8734 $\pm$ 0.0497 & 0.8674 $\pm$ 0.0482 & 0.8722 $\pm$ 0.0487 & 0.8716 $\pm$ 0.0548 \\
1462 & \textbf{1.0000} $\pm$ 0.0000 & \textbf{1.0000} $\pm$ 0.0000 & \textbf{1.0000} $\pm$ 0.0000 & \textbf{1.0000} $\pm$ 0.0000 & \textbf{1.0000} $\pm$ 0.0000 & \textbf{1.0000} $\pm$ 0.0000 & \textbf{1.0000} $\pm$ 0.0000 \\
1464 & 0.7521 $\pm$ 0.0216 & 0.7558 $\pm$ 0.0167 & 0.7555 $\pm$ 0.0188 & 0.7579 $\pm$ 0.0180 & 0.7581 $\pm$ 0.0191 & \textbf{0.7594} $\pm$ 0.0192 & 0.7564 $\pm$ 0.0217 \\
1480 & 0.7355 $\pm$ 0.0441 & 0.7366 $\pm$ 0.0364 & 0.7376 $\pm$ 0.0340 & 0.7382 $\pm$ 0.0351 & 0.7398 $\pm$ 0.0358 & \textbf{0.7406} $\pm$ 0.0359 & 0.7389 $\pm$ 0.0447 \\
1494 & 0.9297 $\pm$ 0.0134 & 0.9325 $\pm$ 0.0153 & 0.9321 $\pm$ 0.0137 & 0.9317 $\pm$ 0.0134 & 0.9324 $\pm$ 0.0132 & 0.9325 $\pm$ 0.0138 & \textbf{0.9339} $\pm$ 0.0135 \\
1510 & 0.9956 $\pm$ 0.0053 & 0.9963 $\pm$ 0.0048 & 0.9965 $\pm$ 0.0048 & 0.9965 $\pm$ 0.0046 & 0.9964 $\pm$ 0.0048 & \textbf{0.9966} $\pm$ 0.0046 & 0.9965 $\pm$ 0.0040 \\
6332 & 0.7753 $\pm$ 0.0331 & 0.7729 $\pm$ 0.0280 & 0.7598 $\pm$ 0.0340 & 0.7609 $\pm$ 0.0316 & 0.7643 $\pm$ 0.0272 & 0.7687 $\pm$ 0.0296 & \textbf{0.7815} $\pm$ 0.0274 \\
23381 & 0.4654 $\pm$ 0.0985 & 0.5101 $\pm$ 0.0819 & \textbf{0.5266} $\pm$ 0.1027 & 0.5204 $\pm$ 0.1110 & 0.5139 $\pm$ 0.0958 & 0.5207 $\pm$ 0.1031 & 0.5092 $\pm$ 0.0954 \\
40966 & \textbf{1.0000} $\pm$ 0.0000 & \textbf{1.0000} $\pm$ 0.0000 & \textbf{1.0000} $\pm$ 0.0000 & \textbf{1.0000} $\pm$ 0.0000 & \textbf{1.0000} $\pm$ 0.0000 & \textbf{1.0000} $\pm$ 0.0000 & \textbf{1.0000} $\pm$ 0.0000 \\
40975 & 0.9723 $\pm$ 0.0182 & 0.9768 $\pm$ 0.0133 & 0.9846 $\pm$ 0.0110 & 0.9798 $\pm$ 0.0130 & 0.9849 $\pm$ 0.0112 & 0.9846 $\pm$ 0.0106 & \textbf{0.9874} $\pm$ 0.0081 \\
40982 & 0.9430 $\pm$ 0.0182 & 0.9480 $\pm$ 0.0177 & 0.9492 $\pm$ 0.0161 & 0.9504 $\pm$ 0.0169 & 0.9515 $\pm$ 0.0165 & 0.9504 $\pm$ 0.0157 & \textbf{0.9553} $\pm$ 0.0132 \\
40994 & 0.9360 $\pm$ 0.0158 & 0.9372 $\pm$ 0.0142 & 0.9398 $\pm$ 0.0120 & 0.9389 $\pm$ 0.0097 & 0.9404 $\pm$ 0.0095 & 0.9393 $\pm$ 0.0080 & \textbf{0.9411} $\pm$ 0.0090 \\
\midrule
Mean & 0.8793 & 0.8858 & 0.8877 & 0.8882 & 0.8882 & \textbf{0.8894} & 0.8890 \\
\bottomrule
\end{tabular}
\end{table*}

\end{document}